\documentclass[11pt]{article}
\usepackage{fullpage}

\usepackage[round]{natbib}

\usepackage[utf8]{inputenc}
\usepackage{amsmath,amsfonts,amsthm, amssymb}
\usepackage{hyperref}
\usepackage[capitalize]{cleveref}
\usepackage{caption}
\usepackage{subcaption}
\usepackage{graphicx}
\usepackage{float}
\usepackage{placeins}
\usepackage{enumitem}
\usepackage[dvipsnames]{xcolor}
\usepackage{bbm}

\newcommand{\dN}{\mathcal{N}}

\newcommand{\diag}{\text{diag}}

\newcommand{\E}{\mathbb{E}}

\newcommand{\kl}{\text{KL}}
\newcommand{\loss}{\mathcal{L}}
\newcommand{\tv}{\text{TV}}
\newcommand{\alg}{\mathcal{A}}
\newcommand{\errgen}{\text{err}_{gen}}

\newtheorem{theorem}{Theorem}
\newtheorem{assumption}[theorem]{Assumption}
\newtheorem{lemma}[theorem]{Lemma}
\newtheorem{corollary}[theorem]{Corollary}
\newtheorem{definition}[theorem]{Definition}

\DeclareMathOperator*{\argmin}{arg\,min}

\begin{document}

%

%


\title{Stability-based Generalization Bounds for Variational Inference}

\author{Yadi Wei and Roni Khardon \\ Department of Computer Science \\ Indiana University, Bloomington \\
             \tt{(weiyadi}$|$\tt{rkhardon)@iu.edu}}

\maketitle

\begin{abstract}
Variational inference (VI) is widely used for approximate inference in Bayesian machine learning. 
In addition to this practical success, 
generalization bounds for variational inference and related algorithms have been developed, mostly through the connection to PAC-Bayes analysis. 
A second line of work
has provided algorithm-specific generalization bounds 
through stability arguments or using mutual information bounds,
and has shown that the bounds are tight in practice,
but unfortunately these bounds do not directly apply to approximate Bayesian algorithms. 
This paper fills this gap by developing algorithm-specific stability based generalization bounds for a class of approximate Bayesian algorithms that includes VI, specifically when using 
stochastic gradient descent to optimize their objective.
As in the non-Bayesian case, 
the generalization error is bounded by
by expected parameter differences on a perturbed dataset.
The new approach complements PAC-Bayes analysis and can provide tighter bounds in some cases.  
An experimental illustration shows that the new approach yields non-vacuous bounds on 
modern neural network architectures and datasets
and that it can shed light on performance differences between variant approximate Bayesian algorithms. 
\end{abstract}

\section{Introduction}
Variational inference (VI \citep{Jordan99}) is one of the most successful approaches in approximate Bayesian machine learning
\citep[e.g.,][]{Blei2003,LimTeh2007,SeegerB12,KingmaW13,Johnson2016} 
and a significant amount of recent work is devoted to variational methods for deep networks
\citep[e.g.,][]{BNN, Practical-VI, DVI, lrt, collapsed-elbo}. 
Instead of calculating an exact posterior one computes an approximate posterior which minimizes 
the KL divergence between the approximation and the true posterior. 
VI is enabled computationally because minimizing the KL divergence is equivalent to maximizing the tractable 
evidence lower bound (ELBO). 
Thanks to this success several variations of VI have been proposed and following \cite{Alquier2016properties} recent work 
has developed finite sample generalization bounds using PAC Bayes analysis. 

This paper continues this effort but from a different perspective, motivated by 
bounds for 
stochastic gradient Langevin dynamics (SGLD)
using Bayes stability \citep{Li,banerjee}.
The idea of analysis through stability \citep{BousquetE02,HardtRS16} is that if an algorithm is not sensitive to perturbations of its input (i.e., the data) then one can bound the gap between its training and test errors. 
\citet{Li,banerjee} employed the KL divergence between output distributions with and without perturbation to assess this sensitivity.
We note that SGLD modifies the parameter $W$ of the learned model, adding noise in the process, but unlike Bayesian algorithms it does not learn a distribution on the parameters. 
Unfortunately, due to this difference, the approach of \citet{Li,banerjee} is not applicable to Bayesian algorithms. 

The paper builds on these ideas and provides a new analysis that 
establishes 
stability-based bounds for a family of approximate-inference algorithms for Bayesian neural networks (which includes VI), in particular when their inference objective is optimized using stochastic gradient descent (SGD). 
We develop two types of bounds: one for bounded loss functions and another for unbounded but Lipschitz loss functions. For bounded loss, we continue to use the KL divergence to measure sensitivity, whereas for Lipschitz loss, we employ the Wasserstein distance. Previous research \citep{ipm, wass-pac-bayes} has explored PAC-Bayes bounds using Wasserstein distance; here, we extend its application to Bayes stability. 
In both cases the generalization gap
can be upper-bounded by the expected parameter differences and further refined using techniques from \citet{HardtRS16} and \citet{SGD-stability2}, resulting in final bounds expressed as the sum of the gradient differences along the optimization trajectory.

\citet{ZhangBHRV17} demonstrated that it is possible to achieve near-zero training error on both true labels (leading to good test performance) and random labels (where test performance is random)
with the same network and training regime. Therefore, any meaningful generalization bound must distinguish between these scenarios, 
implying that it must be data-dependent. 

We provide an empirical demonstration confirming that our bounds can achieve this, effectively differentiating the successful case from overfitting. 
Further, our bounds produce non-vacuous results for generalization error of VI in practical situations, and 
effectively differentiate generalization performance when models are trained with or without data augmentation. 
We also use the bounds to explore the relationship of ELBO to direct loss minimization (DLM \citep{dlm-bnn}), a variant that has shown good performance in other models but fails for Bayesian neural networks, 
showing that the stronger stability of ELBO might explain this performance gap. Finally, a comparison to PAC-Bayes bounds shows that the stability approach provides tighter bounds in these scenarios and can therefore complement the strength of prior analysis.


In summary, this paper introduces a novel approach for analyzing the generalization performance of approximate Bayesian machine learning algorithms. Our contributions include a stability analysis of iterative update algorithms, the application of these bounds to variational Bayesian networks, and empirical demonstration of the practical utility of these bounds.

\section{Related Work}
There is a long tradition of analysis of asymptotic properties of Bayesian algorithms. 
\cite{Alquier2016properties} made an explicit connection between the Gibbs loss used in PAC-Bayes analysis and the objective of VI. This led to finite sample generalization bounds, i.e., bounds on the difference between training and true errors, that hold uniformly. In turn, algorithms that minimize the sum of training error plus generalization bound,
which have a form similar to VI with a regularization parameter,
are both well motivated and have strong theoretical guarantees. 
In followup work
\cite{Germain2016pac,Germain2019} have extended these results to richer classes, whereas \cite{Sheth2017} developed risk bounds, i.e., bounds that directly quantify the true error of VI. 
Other work suggested alternative optimization criteria diverging from VI by changing the loss or regularization components
\cite[e.g.,][]{black-box-alpha,knoblauch2019generalized,dlm-sgp} and generalization and risks bounds have been developed for some such algorithms 
\citep{Sheth2019,Germain2019,Masegosa20,pac-m}. 
However, these have not been demonstrated in practice.
\citet{nonvacuous} provided a non-vacuous bound for a binary classification task on MNIST. 
We evaluate these bounds and compare them to the stability bound
in our experiments in a multi-class classification task with large neural networks.

Another important line of work aims to analyze standard (non-Bayesian) algorithms, where capacity arguments can be used to yield generalization bounds for neural networks \citep[e.g.,][]{nn-rademacher}.
Recent work has developed an alternative approach that provides tighter bounds which are data-dependent and algorithm-dependent. This includes work using stability 
\citep{Li,banerjee}
and analysis that works through bounds on mutual information
\citep{NegreaHDK019,HaghifamNK0D20}. 
This has been specifically developed for SGLD, and extensions to SGD \citep{Neu21} are possible only as an approximation of SGLD.
While the approaches differs in technical details, the outcome is similar in that a generalization bound is obtained which can be expressed as a sum over training steps, of some function of the gradients.
Specifically the bound of \cite{Li} includes a sum of gradient norms whereas the bound of \cite{banerjee} includes a sum of the norms of gradient differences, which was found to be tighter in practice. 
As mentioned above,  
SGLD learns the parameter $W$ of the model and adds some noise duirng the optimization, hence it produces a sample from some distribution over parameters. 
This differs from Bayesian algorithms that explicitly generate distributions over parameters as their posteriors, and aggregate their predictions, and unfortunately the analysis does not carry through to this case. 

In contrast, we directly analyze iterative update Bayesian algorithms, for example, using SGD for variational inference (VI), without noise injection. The primary challenge is that the distribution of the parameters of VI is intractable, making it difficult to apply the chain rule of divergence (as in Lemma 10 of \citet{Li}). 
We provide an alternative analysis that first externalizes all sources of randomness of the algorithm, and then uses convexity to derive the bounds. 
This 
allows us to bound the stability gap in terms of parameter differences.
With this in place we can
follow the approach used to prove the stability of SGD \citep{HardtRS16, SGD-stability2} to bound parameter differences and obtain the desired result.  
Moreover, we extend the original Bayes stability argument, which previously applied only to bounded loss functions \citep{Li} or loss functions with bounded second moments \citep{banerjee}. We generalize this framework to Lipschitz continuous loss functions, allowing us to bound the generalization error using Wasserstein distances, which can be bounded using parameter differences. This extension is inspired by \citet{ipm, wass-pac-bayes}, which employ Wasserstein distances in PAC-Bayes bounds.

Finally,
while the discussion in the paper emphasizes the analysis of VI, 
the analysis and bounds are applicable to any iterative update 
approximate Bayesian 
algorithm that updates parameters of the approximate posterior, where the sensitivity of parameter updates can be easily calculated.
Hence it captures more cases than prior work, as illustrated by the application to DLM.

\section{Preliminaries}
Consider a model with parameters denoted as $w\in \mathbb{R}^d$. Given a prior distribution $p(w)$ and a dataset $S=(z_1, \dots, z_n)$ of size $n$, our goal is to determine the posterior distribution $p(w \mid S)$, which is computationally challenging in most cases. Variational inference offers a solution by seeking a distribution $Q(w)$ from a specified family of distributions, denoted as $\mathcal{Q}$, that minimizes the Kullback-Leibler (KL) divergence between $Q(w)$ and the true posterior $p(w|S)$. 
\begin{align}
\label{eq:elbo}
&Q^*(w) \nonumber \\
=& \argmin_{Q\in \mathcal{Q}}\kl(Q(w) \Vert p(w \mid S)) \nonumber \\
=& \argmin_{Q\in \mathcal{Q}} \E_{Q(w)}[\log Q(w) - \log p(w, S)] + \log p(S) \nonumber \\
=&\argmin_{Q\in \mathcal{Q}} 
\frac{1}{n} 
\sum_{i=1}^n 
\E_{Q(w)}[-\log p(z_i|w)] + \frac{1}{n} \kl(Q, p). 
\end{align}
The maximization objective obtained by negating \eqref{eq:elbo} is known as the Evidence Lower Bound (ELBO).
The above optimization objective can be efficiently solved using common gradient-based techniques, such as stochastic gradient descent. Furthermore, various alternative objectives exist to approximate the (pseudo) posterior distributions, for example, Direct Loss Minimization (DLM, \citep{dlm-sgp}), which uses the the following objective, and which is discussed in our experiments:
\begin{align}
\label{eq:dlm}
    \frac{1}{n}\sum_{i=1}^n -\log \E_{Q(w)}[p(z_i \vert w)] + \frac{1}{n} \kl(Q, p).
\end{align}

Note that the optimization objective is a function of the distribution $Q(w)$ and let $\theta$ denote the parameters of $Q$. 
To facilitate the analysis across different objectives,
we denote the objective function as $F(\theta, S) = \frac{1}{n} \sum_{i=1}^n F(\theta, z_i)$, where the objective function is written as the average of the objective function with respect to individual data points.  
Notice that for the examples above $F$ includes the regularizer.
For example, in ELBO, $F(\theta, z) = \E_{Q(w)}[-\log p(z|w)] + \frac{1}{n} \kl (Q, p)$. 

Let $\loss(w, z)$ be a loss function for parameter $w$ on a data point $z$ (notice that the loss can be different from the objective function $F$). Define $\loss(w, \mathcal{D}) = \E_{z \sim D} [\loss(w, z)]$ as the expected loss over a distribution $\mathcal{D}$, and $\loss(w, S) = \frac{1}{n} \sum_i \loss(w, z_i)$ as the empirical loss on a dataset $S$. Then the generalization error of the algorithm $\alg$ (which chooses $Q$ based on $S$), i.e., the gap between true and training set error, is given by:
\begin{align}
    \errgen(\alg) &= \E_{S\sim \mathcal{D}^n} \E_{w\sim Q} [\loss(w, \mathcal{D}) - \loss(w, S)].
\end{align}

\section{Generalization Bounds through Bayes Stability}
\label{sec:stability}
Consider a Bayesian algorithm, denoted by $\mathcal{A}$, designed to learn the posterior distribution over a parameter $w$ by optimizing an objective function $F(\theta, S)$. In some cases, there is inherent randomness in evaluating the objective and its gradients or in the optimization process, such as when the reparametrization trick \citep{KingmaW13} is used to approximate expectation terms (as in \cref{eq:elbo} and \cref{eq:dlm}) or when mini-batches are employed. We represent all sources of randomness by $\epsilon$. Consequently, the gradient of the objective becomes $\nabla F(\theta, S, \epsilon)$ for the entire dataset and $\nabla F(\theta, z, \epsilon)$ for an individual data point $z$.
Given a training dataset $S$ and the randomness $\epsilon$, the algorithm $\mathcal{A}$ deterministically produces a posterior distribution $Q_{S, \epsilon}$ for the parameter $w$. We define $Q_S$ as the expected posterior distribution, obtained by averaging over all possible randomness, i.e., $Q_S = \mathbb{E}_{\epsilon}[Q_{S, \epsilon}]$. Additionally, we assume that, when $\epsilon$ is integrated out, $\mathcal{A}$ is order-independent.

\begin{assumption}[Order-independent]
    For any fixed dataset $S=(z_1, \dots, z_n)$ and any permutation $p$, $Q_S = Q_{S^p}$, where $S^p$ is the dataset under permutation $p$.
\end{assumption}
This assumption can be easily satisfied by letting the learning algorithm randomly permute the training data at the beginning. Additionally, it is straightforward to show that variational inference using stochastic gradient descent (SGD) satisfies this condition.

We proceed, following the work of \citet{Li}, to define the single-point posterior distribution $Q_z = \E_{(z_1, \dots, z_{n-1})} [Q_{(z_1, \dots, z_{n-1}, z)}] = \E_{\epsilon, (z_1, \dots, z_{n-1})}[Q_{(z_1, \dots, z_{n-1}, z), \epsilon}]$, 
where we assume without loss of generality that $z$ is put at location $n$. 

\subsection{Bayes Stability}

The generalization error can be effectively bounded using a Bayes-stability argument, as exemplified in previous work by \citet{Li} and \citet{banerjee}. 
We develop two such bounds, one for bounded loss functions using TV distance and the other for unbounded but Lipschitz loss functions using Wasserstein distance. In both cases the result reduces to expected parameter differences. 

Let $\tv(p, q) = \frac{1}{2} \int \vert p(x) - q(x) \vert dx$ be the total variation distance between distributions. We have:

\begin{lemma}
\label{lemma:tv-convexity}
    $\tv\left(\E_{P(X)}[P(Y|X)], \E_{P(X)}[Q(Y|X)] \right) \leq \E_{P(X)}[\tv(P(Y|X), Q(Y|X))]$.
\end{lemma}
\begin{proof}
    \begin{align*}
        & \tv(\E_{P(X)}P(Y|X), \E_{P(X)}Q(Y|X)) 
        \\ &= \frac{1}{2} \int \left|\int P(x) P(y|x)dx - \int P(x) Q(y|x) dx\right| dy
        \\ &= \frac{1}{2} \int \left|\int P(x) (P(y|x)- Q(y|x)) dx\right| dy 
        \\ &\leq \frac{1}{2} \int P(x) \int \left|P(y|x)- Q(y|x)\right| dy dx 
        \\ &= \E_{P(X)}[\tv(P(Y|X), Q(Y|X))].
    \end{align*}
\end{proof}

The following lemma 
adapts the ideas in the original proofs of \citet{Li, banerjee} to the context of Bayesian algorithms that output distributions over parameters. 

\begin{lemma}[Bayes-Stability 1]
\label{thm:errgen}
    Suppose the loss function $\loss(w, z)$ is $C$-bounded. Let $S$, $\Bar{S}$ denote two datasets that only differ at one element $z$ and $\Bar{z}$. The generalization error $\errgen(\alg)$ is upper bounded by $2C \E_{S, \Bar{S}, \epsilon} [\tv(Q_{S, \epsilon}, Q_{\Bar{S}, \epsilon})]$.
\end{lemma}
\begin{proof}
    It is clear that
     \begin{align}
         \E_{S} \E_{w \sim Q_S} [\loss(w, \mathcal{D})] = \E_{z\sim \mathcal{D}} \E_{w \sim Q_z} \E_{\Bar{z}\sim \mathcal{D}}[\loss(w, \Bar{z})] = \E_{\Bar{z}} \E_{Q_{\Bar{z}}} \E_{z} [\loss(w, z)],
     \end{align}
     and
     \begin{align}
         \E_S \E_{w\sim Q_S} \left[\frac{1}{n} \sum_{i=1}^n \loss(w, z_i) \right] &= \E_{z} \E_{Q_z} [\loss(w, z)].
     \end{align}
     Then the generalization error is 
     \begin{align}
         \errgen(\alg) &= \E_z \E_{\Bar{z}} \left[ \E_{w\sim Q_{\Bar{z}}}[\loss(w, z)] - \E_{w\sim Q_z}[\loss(w, z)]\right] \\
         &\leq \E_{z, \Bar{z}} \int |\loss(w, z)| \lvert Q_{\Bar{z}}(w) - Q_z(w)\rvert dw \\
         &\leq 2C \E_{z, \Bar{z}} [\tv(Q_z, Q_{\Bar{z}})] \\
         &= 2C \E_{z, \Bar{z}}[\tv(\E_{S_{n-1}, \epsilon}[Q_{S_{n-1} \cup \{z\}, \epsilon}], \E_{S_{n-1}, \epsilon}[Q_{S_{n-1} \cup \{\Bar{z}\}, \epsilon }])] \\
         &\leq 2C \E_{z, \Bar{z}, S_{n-1}, \epsilon}[\tv(Q_{S, \epsilon}, Q_{\Bar{S}, \epsilon})]
     \end{align}
      where the last inequality is due to \cref{lemma:tv-convexity}.
\end{proof}

There are two important differences form the argument structure in prior work \citep{Li, banerjee}. 
First, note that it is crucial that $\epsilon$ includes all sources of randomness in the algorithm. With this condition, $Q_{S, \epsilon}$ is a distribution in the family used by the algorithm and not a mixture of such distributions.
For example, when $Q(w)$ is a normal distribution $Q_{S, \epsilon}$ is a normal distribution, but $Q_{S}$ is a mixture of normal distributions where the mixture is taken over $\epsilon$. 
This allows us to directly bound the stability using parameter differences as in the next lemma. In contrast, the analysis of \citet{Li, banerjee}, 
that works with mixtures generated by the choice of batches in SGLD,
requires a fixed variance term (for all dimensions) and is not easily generalizable to the case of learned variances. 

The second difference is due to the structure of the probability model.
\citet{Li} 
use the sum of KL divergence along the optimization trajectory to upper bound the Bayes stability. In SGLD, the optimization trajectory \( W_1, W_2, \dots, W_T \) consists of samples from the distribution, with each \( W_i \) being drawn from a Gaussian distribution conditioned on both \( W_{i-1} \) and the batch. However, in variational inference, the optimization trajectory \( (\mu_1, \sigma_1), \dots, (\mu_T, \sigma_T) \) consists of distribution parameters, 
and the bound on the sequence of conditional KL divergences does not hold. 
A detailed explanation is provided in \cref{sec:Li-proof-not-for-vi}.

The next lemma shows that Bayes stability can be bounded through parameter differences:

\begin{lemma}
\label{cor:kl}
    Under the condition of \cref{thm:errgen}, if $Q_{S, \epsilon} = \dN(m, \diag(\sigma^2))$ and $Q_{\Bar{S}, \epsilon} = \dN(\Bar{m}, \diag(\Bar{\sigma}^2))$, the generalization error is upper bounded by 
    \begin{align}
    \label{eq:kl-bound}
        \frac{2C}{\sqrt{\sigma_0}} \sqrt{\E [\lVert \Bar{\sigma} - \sigma \rVert_1]} + \frac{C}{\sigma_0} \E \left[\lVert \Bar{\sigma} - \sigma \rVert_2 \right] + \frac{C}{\sigma_0} \E \left[\lVert \Bar{m} - m \rVert_2 \right],
    \end{align}
    where the expectation is taken over $S, \Bar{S}$ and $\epsilon$ and $\sigma_0$ is a preset lower bound of the standard deviation in $Q$.
\end{lemma}
\begin{proof}
    According to Pinsker's inequality, the total variation distance can be bounded by the KL divergence of the distributions. We thus first bound the KL divergence. 
     \begin{align}
         \kl(Q_{S, \epsilon}, Q_{\Bar{S}, \epsilon}) &= 1^\top (\log \sigma - \log \Bar{\sigma}) + \frac{1}{2} \left(1^\top \frac{\Bar{\sigma}^2}{\sigma^2} - d + 1^\top \frac{(\Bar{m} - m)^2}{\sigma^2} \right) \\
        &\leq \frac{2\lVert \sigma - \Bar{\sigma} \rVert_1}{\sigma_0} + \frac{\lVert \sigma - \Bar{\sigma} \rVert_2^2}{2\sigma_0^2} + \frac{\lVert \Bar{m} - m \rVert_2^2}{2\sigma_0^2},
        \label{eq:kl-form}
    \end{align}
    where $\sigma_0$ is a preset minimum value for the standard deviation, i.e., $\forall k, \sigma_k \geq \sigma_0$ and $\bar{\sigma}_k \geq \sigma_0$.
    To derive \cref{eq:kl-form},
    let $ \beta_i = |\sigma_i - \Bar{\sigma}_i |$. Consider $1^\top (\log \sigma - \log \Bar{\sigma}) = \sum_i \log \frac{\sigma_i}{\Bar{\sigma}_i}$. For each entry $i$, if $\sigma_i - \beta_i \leq \sigma_0$, then $\log \frac{\sigma_i}{\Bar{\sigma}_i} \leq \log \frac{\beta_i + \sigma_0}{\sigma_0} = \log (1 + \frac{\beta_i}{\sigma_0}) \leq \frac{\beta_i}{\sigma_0}$; if $\sigma_i - \beta_i > \sigma_0$, then $\log \frac{\sigma_i}{\Bar{\sigma}_i} \leq \log \frac{\sigma_i}{\sigma_i - \beta_i} = \log (1 + \frac{\beta_i}{\sigma_i - \beta_i}) \leq \log (1 + \frac{\beta_i}{\sigma_0}) \leq \frac{\beta_i}{\sigma_0}$. Overall, $1^\top (\log \sigma - \log \Bar{\sigma}) \leq \sum_i \frac{\beta_i}{\sigma_0} = \frac{\lVert \sigma - \Bar{\sigma}_0 \rVert_1}{\sigma_0}$. For $1^\top \frac{\Bar{\sigma}^2}{\sigma^2} = \sum_i \frac{\Bar{\sigma}_i^2}{\sigma_i^2} \leq \sum_i \frac{(\sigma_i + \beta_i)^2}{\sigma_i^2} \leq \sum_i (1 + 2\frac{\beta_i}{\sigma_i} + \frac{\beta_i^2}{\sigma_i^2}) \leq \sum_i (1 + \frac{2\beta_i}{\sigma_0} + \frac{\beta_i^2}{\sigma_i}) = d + \frac{2\lVert \sigma - \Bar{\sigma} \rVert_1}{\sigma_0} + \frac{\lVert \sigma - \Bar{\sigma} \rVert_2^2}{\sigma_0^2}$. Thus,
    \begin{align}
        \errgen(\alg) &\leq 2C \E_{S, \Bar{S}, \epsilon} [\tv(Q_{S, \epsilon}, Q_{\Bar{S}, \epsilon})] \\
        &\leq C \E_{S, \Bar{S}, \epsilon} \sqrt{2 \kl (Q_{S, \epsilon}, Q_{\Bar{S}, \epsilon})} \\
        &\leq \frac{2C}{\sqrt{\sigma_0}} \sqrt{\E [\lVert \Bar{\sigma} - \sigma \rVert_1]} + \frac{C}{\sigma_0} \E \left[\lVert \Bar{\sigma} - \sigma \rVert_2 \right] + \frac{C}{\sigma_0} \E \left[\lVert \Bar{m} - m \rVert_2 \right].
    \end{align}
\end{proof}

\cref{thm:errgen} holds only for bounded loss functions. 
We next introduce the upper bound for unbounded Lipschitz loss functions using Wasserstein distance \citep{Villani2008OptimalTO}.
\begin{definition}
    Suppose loss function $\loss(\theta, z)$ is $K$-Lipschitz with respect to $\theta$, i.e., $\frac{|\loss(\theta, z) - \loss(\theta', z)|}{\lVert \theta - \theta' \rVert} \leq K$ for all $z$.
\end{definition}
The Wasserstein-$p$ distance between two distributions $\mu$ and $\nu$ is defined as:
\begin{align}
    W_p(\mu, \nu) = \inf_{\gamma \in \Gamma(\mu, \nu)} \left( \E_{(x, y) \sim \gamma} d(x, y)^p \right)^{1/p},
\end{align}
where $d(x,y)$ is some distance and $\Gamma(\mu, \nu)$ is the set of all couplings of $\mu$ and $\nu$, i.e., for $\gamma \in \Gamma(\mu, \nu)$, $\int_y \gamma(x, y) = \mu(x)$ and $\int_x \gamma(x, y) = \nu(y)$. 
In the following we use the Euclidean distance.
According to Kantorovich duality \citep{Villani2008OptimalTO},
\begin{align}
    W_1(\mu, \nu) = \sup_{f, \text{Lip}(f)\leq 1} \int f d\mu(x) - \int f d\nu(y).
\end{align}
Inspired by \citet{ipm}, we derive the bound through the Wasserstein distance:
\begin{lemma} [Bayes-Stability 2]
\label{thm:wasserstein}
    Suppose the loss function $\loss(w, z)$ is $K$-Lipschitz. Let $S$, $\Bar{S}$ denote two datasets that only differ at one element $z$ and $\Bar{z}$. 
    Then, for any $p \geq 1$, the generalization error $\errgen(\alg)$ is upper bounded by 
    \begin{align}
        K \E_{S, \Bar{S}, \epsilon} [W_p(Q_{S, \epsilon}, Q_{\Bar{S}, \epsilon})].
    \end{align}
\end{lemma}
\begin{proof}
    It is obvious that $\frac{1}{K} \loss(w, z)$ is 1-Lipschitz. 
    Using Kantorovich duality we have
    \begin{align*}
         &\errgen(\alg) \\
         &= \E_z \E_{\Bar{z}} \left[ \E_{w\sim Q_{\Bar{z}}}[\loss(w, z)] - \E_{w\sim Q_z}[\loss(w, z)]\right] \\
         &\leq K \E_{z, \Bar{z}} \sup_{f, \text{Lip}(f)\leq 1} \E_{w\sim Q_{\Bar{z}}}[f(w)] - \E_{w\sim Q_{z}}[f(w)] \\ 
         &\leq K \E_{z, \Bar{z}} \E_{S_{n-1}, \epsilon} \\
         & \qquad \qquad \sup_{f, \text{Lip}(f)\leq 1}  \E_{w \sim Q_{\Bar{S}, \epsilon}}[f(w)] - \E_{w \sim Q_{S, \epsilon}}[f(w)] \\
         &= K \E_{z, \Bar{z}, S_{n-1}, \epsilon} W_1(Q_{\bar{S}, \epsilon}, Q_{S, \epsilon}) \\
         &\leq K \E_{z, \Bar{z}, S_{n-1}, \epsilon} W_p (Q_{\bar{S}, \epsilon}, Q_{S, \epsilon}).
    \end{align*}
    The third line is because of the convexity of supremum. The last inequality follows the Holder's inequality, which states that $\E[|XY|] \leq \E[|X|^p]^{\frac{1}{p}} \E[|Y|^q]^{\frac{1}{q}}$ for $p, q \geq 1$ and $\frac{1}{p} + \frac{1}{q} = 1$. Thus $\E_{x, y \sim \gamma}[d(x, y) \cdot 1] \leq \left(\E_{x, y \sim \gamma}[d(x, y)^p]\right)^{\frac{1}{p}} \left( \E[1^{q}] \right)^{\frac{1}{q}} = \left(\E_{x, y \sim \gamma} d(x, y)^p \right)^{\frac{1}{p}}$. 
    Taking infimum on both sides, we have proved the inequality.
\end{proof}

As in the previous case we can bound the stability using parameter differences. In particular, 
letting $p=2$ and using the Wasserstein-2 distance for Gaussian distributions \citep{Gauss-wass}, we immediately have: 
\begin{lemma}
\label{thm:wasserstein2}
Under the condition of \cref{thm:wasserstein},
if $Q_{S, \epsilon} = \dN(m, \diag(\sigma^2))$ and $Q_{\Bar{S}, \epsilon} = \dN(\Bar{m}, \diag(\Bar{\sigma}^2))$, the generalization error is upper bounded by 
    \begin{align}
    \label{eq:wass-bound}
        K \E \lVert m - \Bar{m} \rVert_2 + K \E \lVert \sigma - \Bar{\sigma} \rVert_2.
    \end{align}
\end{lemma}


\subsection{Bounds on Expected Parameter Differences}
\label{sec:bounds}
In this section, we draw upon the approach from \citet{HardtRS16} and \citet{SGD-stability2}, which bounds parameter differences for stochastic gradient descent.
Let $\theta_t$ be the parameter of $Q_{S, \epsilon}$ at step $t$ and $\Bar{\theta}_t$ be the parameter of $Q_{\Bar{S}, \epsilon}$ at step $t$. Let $G_t$ denote the update rule of stochastic gradient descent with learning rate $\alpha_t$,
\begin{align}
    \theta_{t} = G_t(\theta_{t-1}, S, \epsilon_{t}) = \theta_{t-1} - \alpha_t \nabla_\theta F(\theta_{t-1}, S, \epsilon_{t}).
\end{align}
Recall that $\epsilon_t$ contains all randomness at step $t$ and $\nabla F(\theta_{t-1}, S, \epsilon_t)$ is the approximation of $\nabla F(\theta_{t-1}, S)$. We make the following assumption \citep{HardtRS16,SGD-stability2} on the update rule:
\begin{definition}
    An update rule is $\eta$-expansive if $\sup_{\theta, \theta'}\frac{\lVert G(\theta, S, \epsilon) - G(\theta', S, \epsilon) \rVert}{\lVert \theta - \theta' \rVert} \leq \eta$ for any $S$ and $\epsilon$.
\end{definition}
The following theorem adapts the argument of \citet{SGD-stability2} to bound parameter differences as a function of expected gradient differences.
\begin{theorem}
\label{thm:diff}
    Given an algorithm that optimizes parameters $\theta$ using stochastic gradient descent, suppose it is $\eta_t$-expansive for step $t$. Let $S$ and $\Bar{S}$ be two random datasets of size $n$ that only differ at one element $z$ and $\Bar{z}$, and $\theta_T$ and $\Bar{\theta}_T$ denote the outputs under the same $\epsilon$.
    Then the expected difference of $\theta_T$ and $\Bar{\theta}_T$ satisfies 
    \begin{align}
    \label{eq:diff}
        \E_{S, \Bar{S}, \epsilon} [\lVert \theta_T - \Bar{\theta}_T \rVert]
        \leq \frac{1}{n} \sum_{t=1}^T \left(\prod_{i=t+1}^T \eta_i \right) \alpha_t \E_{S, \epsilon, \Bar{z}}[\Delta_t],
    \end{align}
    where $\Delta_t = \lVert \nabla F(\theta_{t-1}, \Bar{z}, \epsilon_t) - \nabla F(\theta_{t-1}, z, \epsilon_t) \rVert$.
\end{theorem}
\begin{proof}
    Let $S_t$ and $\Bar{S}_t$ denote the subset at step $t$. With respect to the same $\epsilon_t$ (including the same batch sequence), $S_t$ and $\bar{S}_t$ have at most one different element.
    We have two cases:
    \begin{itemize}
    \item Case 1: the different element is not selected, hence $S_t = \Bar{S}_t$, and since $G$ is $\eta_t$ expansive:
        \begin{align*}
            \lVert \theta_t - \Bar{\theta}_t \rVert \leq \eta_t \lVert \theta_{t-1} - \Bar{\theta}_{t-1} \rVert.
        \end{align*}
        \item Case 2: the different element is selected.
        \begin{align*}
            \lVert \theta_t - \Bar{\theta}_t \rVert &= \lVert (\theta_{t-1} - \alpha_t \nabla F(\theta_{t-1}, S_t, \epsilon_{t})) \\
            & \quad - (\Bar{\theta}_{t-1} - \alpha_t \nabla F(\Bar{\theta}_{t-1}, \Bar{S}_t, \epsilon_t)) \rVert \\
            &= \lVert (\theta_{t-1} - \alpha_t \nabla F(\theta_{t-1}, \Bar{S}_t, \epsilon_t)) \\
            & \quad - (\Bar{\theta}_{t-1} - \alpha_t \nabla F(\Bar{\theta}_{t-1}, \Bar{S}_t, \epsilon_t)) \\
            & \quad + \alpha_t (\nabla F(\theta_{t-1}, \Bar{S}_t, \epsilon_t) - \nabla F(\theta_{t-1}, S_t, \epsilon_t)) \rVert \\
            &\leq \eta_t \lVert \theta_{t-1} - \Bar{\theta}_{t-1} \rVert \\
            &\quad + \alpha_t \lVert \nabla F(\theta_{t-1}, \Bar{S}_t, \epsilon_t) - \nabla F(\theta_{t-1}, S_t, \epsilon_t) \rVert.
        \end{align*}
        Since $\Bar{S}_t$ and $S_t$ only differs at one element, $\lVert \nabla F(\theta_{t-1}, \Bar{S}_t, \epsilon_t) - \nabla F(\theta_{t-1}, S_t, \epsilon_t) \rVert = \frac{1}{b} \lVert \nabla F(\theta_{t-1}, \Bar{z}, \epsilon_t) - \nabla F(\theta_{t-1}, z, \epsilon_t) \rVert = \frac{1}{b} \Delta_t$, where $b$ is the batch size.
    \end{itemize}
    Thus, 
    \begin{align}
        \lVert \theta_T - \Bar{\theta}_T \rVert &\leq \eta_T \lVert \theta_{T-1} - \Bar{\theta}_{T-1} \rVert + \mathbbm{1}_{z \in S_T} \frac{\alpha_T}{b} \Delta_T \\
        &\leq \frac{1}{b} \sum_{t=1}^T \left(\prod_{i=t+1}^T \eta_i \right) \mathbbm{1}_{z\in S_t} \alpha_t \Delta_t,
    \end{align}
    where the base case is $\theta_0 = \bar{\theta}_0$.
    Since the probability that $z \in S_t$ is $\frac{b}{n}$, then the expected difference is 
    \begin{align}
        \E \lVert \theta_T - \Bar{\theta}_T \rVert \leq \frac{1}{n} \sum_{t=1}^T \left(\prod_{i=t+1}^T \eta_i \right) \alpha_t \E_{S, \epsilon, \bar{z}} [\Delta_t].
    \end{align}
\end{proof}

\cref{thm:diff} provides a way to compute the bound {\em exactly}. As we show in the experiments this allows us to obtain tight generalization bounds which are not possible otherwise. For completeness, the following Corollary provides an asymptotic upper bound using stronger requirements.
The proof follows the construction of \citet{SGD-stability2} and is included in \cref{sec:proof-bound}.
\begin{corollary}
\label{cor:logT}
    Suppose $\nabla F(\theta, S, \epsilon)$ is $L$-Lipschitz and $\beta$-bounded, then with learning rate $\alpha_t = \frac{c}{(t+2) \log (t+2)}$ where $c$ is chosen that $cL < 1$, $\E \lVert \theta_T - \Bar{\theta}_T \rVert \leq O(\frac{\log T}{n})$.
\end{corollary}

\subsection{Discussion: Stability vs.\ PAC-Bayes Bounds}
\label{sec:pac-bayes}

As mentioned above, prior work has developed  PAC-Bayes Bounds for certain variants of VI. In this section we review some of these bounds and discuss the qualitative differences between the two types of bounds.

\citet{Germain2016pac} provides a generalization error bound for a $C$-bounded loss function as follows: with probability $1 - \delta$,
\begin{align}
    \frac{1}{\lambda} \left(\kl(Q_S \parallel P) + \log \frac{1}{\delta}\right) + \frac{\lambda C^2}{2n}.
\end{align}
By optimizing $\lambda$ as $\lambda = \frac{1}{C} \sqrt{2n \left(\kl(Q_S \parallel P) + \log \frac{1}{\delta}\right)}$, we obtain the following bound:
\begin{align}
    C \sqrt{\frac{2 \left(\kl(Q_S \parallel P) + \log \frac{1}{\delta}\right)}{n}}.
\end{align}

On the other hand, \citet{pac-bayes-book} provides a similar bound in the form:
\begin{align}
\label{eq:sqrt}
    C\sqrt{\frac{\kl(Q_S \parallel P) + \log \frac{n}{\delta}}{2(n-1)}}.
\end{align}
that can be tighter in some cases.
In these results, 
the ``prior" $P$ is only required to be data independent and is not directly related to the algorithm. 
Therefore, for Bayesian algorithms, one can pick a different $P$ other than the prior used in the objective function. 
In the experimental illustration,
we explore using both the prior and the initialization $Q_0$ so as to obtain the tightest possible bound.

Additionally, \citet{nonvacuous} proposed a non-vacuous bound specifically for the 0-1 loss. By employing a union-bound argument, 
where the prior variance is set as $\lambda = c \exp(-j / b)$ for $j \in \mathbb{N}$ and fixed $b$ and $c$,
they ensure that the generalization error can be bounded, with probability $1-\delta$, by
\begin{align}
\label{eq:bre}
    \sqrt{\frac{\kl(Q_S \parallel \mathcal{N}(m_0, \lambda I)) + 2 \log \left(b \log \frac{c}{\lambda}\right) + \log \frac{\pi^2 n}{6 \delta}}{2(n-1)}},
\end{align}
where $m_0$ denotes the random initialization of the mean parameter.

These bounds have been leveraged to develop efficient Bayesian algorithms, by explicitly optimizing the sum of the training set loss and the bound, 
which can be seen to have a similar form to VI and therefore interpreted as variants of VI.
On the other hand, PAC-Bayes bounds are valid for any distribution within the specified family. 
They can therefore be applied to the output of VI directly. 

From this perspective, our bounds are more restricted in that they are valid only for the output of  a certain class of optimization problems when optimized by SGD. In addition, the stability bound in \eqref{eq:diff} grows with the number of optimization steps $T$ which can make it less attractive,
and for a fixed dataset this may necessitate the use of larger batch sizes to reduce $T$.
On the other hand, the dependence on dataset size in \eqref{eq:diff} is $\frac{1}{n}$ whereas the one in the PAC-Bayes bounds is $\frac{1}{\sqrt{n}}$ so our bound has the potential to be tighter for large datasets. 
Appendix~\ref{app:pacbayes} shows an example where the PAC-Bayes bound can grow arbitrarily in a case where the stability bound is tight. 
Overall, the two approaches can have advantages in different situations and both contribute to our understanding of generalization performance of algorithms.

\section{Experimental Illustration}

In this section we explore the potential of stability based bounds to capture generalization error 
and compare them to PAC Bayes bounds.
We also evaluate the expansion rate that appears in the bound showing that it can be small, and hence better in practice than the use of the asymptotic bounds. 

We adopted the experimental setup used by \citet{Li} and \citet{banerjee} and conducted our experiments on CIFAR10 using the same CNN model that has been employed in these works. 
For algorithms, 
we use the ELBO (\cref{eq:elbo}) and DLM variant (\cref{eq:dlm}) with a KL divergence coefficient of 0.1, a value that has been demonstrated to yield superior results in previous studies \citep[e.g.,][]{cold-posterior}.
Our optimization was performed using the SGD optimizer with an initial learning rate of 0.005, momentum of 0.99, and we reduced the learning rate by a factor of 0.9 every 5 epochs thereafter. We select the batch size to be 1000 and set $\sigma_0=0.01$. 
All experiments are run on a single NVIDIA Tesla V100 PCIe 32 GB GPU.

We perform two sets of experiments. 
In the first we test the performance of ELBO 
with or without data augmentation (random cropping and horizontal flipping \citep{shorten2019survey}) as well as random label perturbations,
comparing the generalization error (measured by 0-1 loss) and our bound (\cref{eq:kl-bound}) with $C=1$ under these situations. 
In the second, 
following the observation by \citet{dlm-bnn} that DLM (\cref{eq:dlm}) does not perform as well as ELBO in Bayesian neural networks,
we use the bounds to compare ELBO and DLM in terms of log loss.

The primary goals of our experiments were to demonstrate the following key points.
Our bound is non-vacuous in successful learning cases and becomes 
vacuous when the dataset contains a sufficiently high proportion of random labels.
In addition, our bound accurately reflects the reduction in generalization error with data augmentation.
Finally, our bound can potentially provide an explanation for the failure of DLM, suggesting that its lower stability might be the cause of higher generalization error.
For these experiments, the stability bound is both tighter and has more explanatory power than the PAC-Bayes bounds, hence demonstrating the utility of the new derivations.


\begin{figure*}[t]
    \centering
    \begin{subfigure}[b]{0.40\textwidth}
         \centering
         \includegraphics[width=\textwidth]{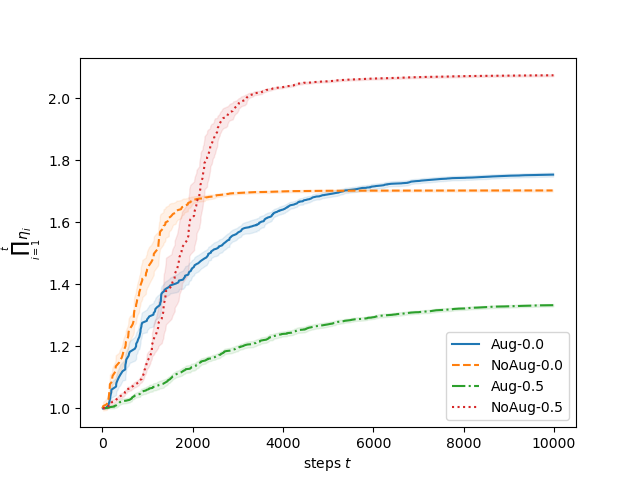}
    \end{subfigure}
    \begin{subfigure}[b]{0.40\textwidth}
         \centering
         \includegraphics[width=\textwidth]{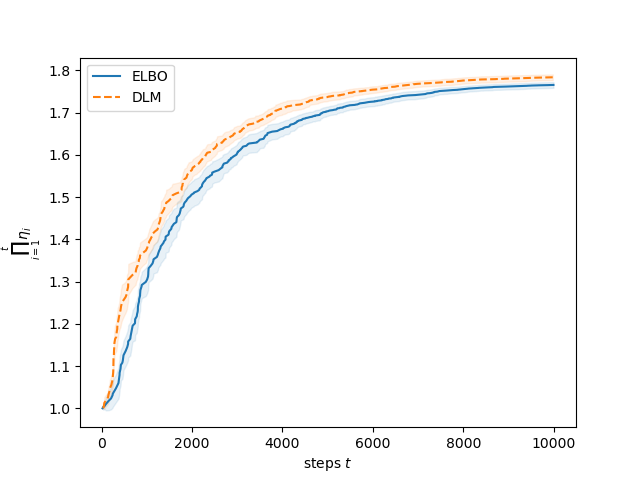}
    \end{subfigure}
    \caption{Cumulative expansion rates under various conditions. The left panel displays expansion rates with and without data augmentation, comparing cases with random labels (50\% random, labeled as 0.5) and without random labels (labeled as 0.0). The right panel shows expansion rates across different algorithms with data augmentation and no random labels. The shaded areas represent the standard deviation across 10 runs.}
    \label{fig:expansion}
\end{figure*}

\paragraph{Expansion Rate}
We start by evaluating the expansion rate which is needed for the exact bound. 
To perform this,
we randomly initialize two models and then run the same algorithm with the same batch sequence. We keep track of the norm of the parameter difference and compute the expansion rate at each step $t$. 
For simplicity, we take the maximum of the expansion rate of both $m$ and $\sigma$ (both $L_1$-norm and $L_2$-norm).

\cref{fig:expansion} shows the cumulative expansion rate under various conditions. It is evident that for each method the expansion rate increases more slowly as the number of steps increases, and the final rate shows minimal variance. 
We observe that without data augmentation the expansion rate quickly levels off. This occurs because the dataset is straightforward to learn, and once all data has been learned, the gradient approaches zero, causing the expansion rate to flatten. In contrast, with data augmentation, the expansion rate continues to grow. 
We also observe that the expansion rate of DLM is slightly higher than that of ELBO.

For use in evaluating generalization bounds, 
we note that the final cumulative expansion rate is much smaller than the $\log T$ factor in \cref{cor:logT} in all cases and will therefore lead to tighter bounds in practice.
We therefore run this evaluation 10 times and use the mean value plus four standard deviation as the final value $\eta_t$.

\begin{figure*}[t]
    \centering
        \begin{subfigure}[b]{0.32\textwidth}
         \centering
         \includegraphics[width=\textwidth]{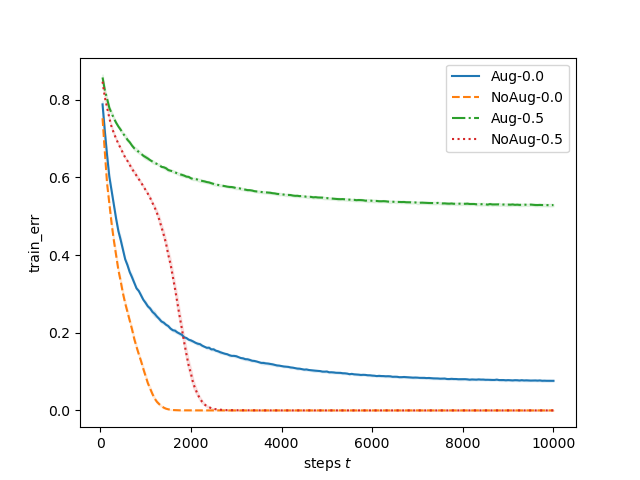}
         \caption{Train error.}
    \end{subfigure}
    \begin{subfigure}[b]{0.32\textwidth}
         \centering
         \includegraphics[width=\textwidth]{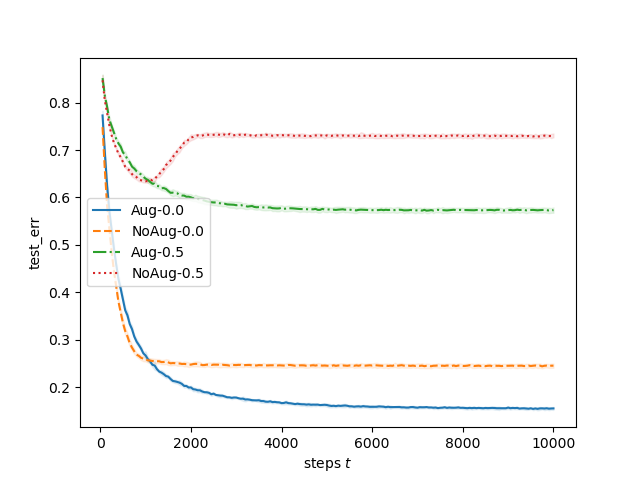}
         \caption{Test error.}
    \end{subfigure}
    \begin{subfigure}[b]{0.32\textwidth}
         \centering
         \includegraphics[width=\textwidth]{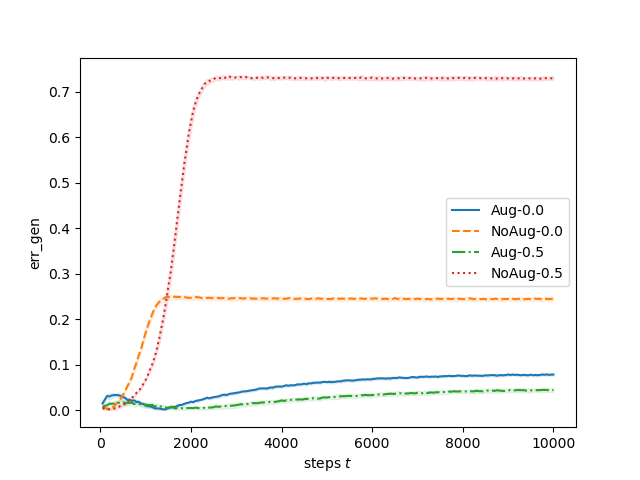}
         \caption{Generalization error.}
    \end{subfigure}
    \begin{subfigure}[b]{0.32\textwidth}
         \centering
         \includegraphics[width=\textwidth]{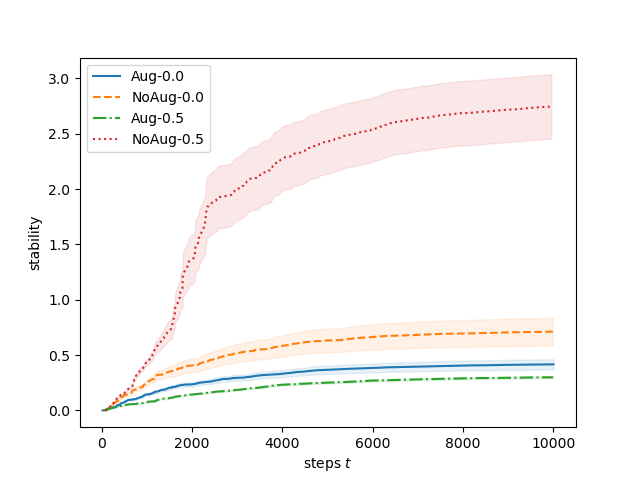}
         \caption{Stability Bound (\ref{eq:kl-bound}).}
    \end{subfigure}
    \begin{subfigure}[b]{0.32\textwidth}
         \centering
         \includegraphics[width=\textwidth]{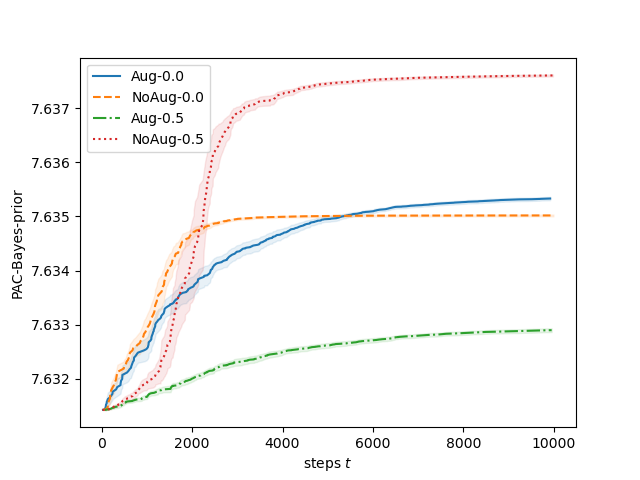}
         \caption{PAC-Bayes (\ref{eq:sqrt}) with prior.}
    \end{subfigure}
    \begin{subfigure}[b]{0.32\textwidth}
         \centering
         \includegraphics[width=\textwidth]{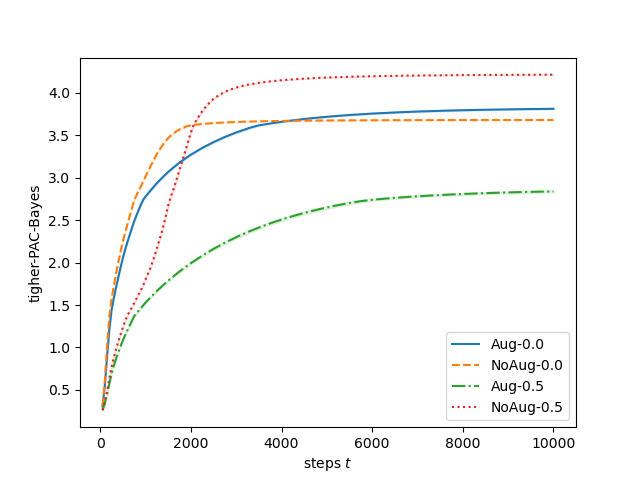}
         \caption{Tighter PAC-Bayes (\ref{eq:bre}).}
    \end{subfigure}
    \caption{Generalization error and bounds.}
    \label{fig:bound}
\end{figure*}


\begin{figure*}[h]
    \centering
        \begin{subfigure}[b]{0.32\textwidth}
         \centering
         \includegraphics[width=\textwidth]{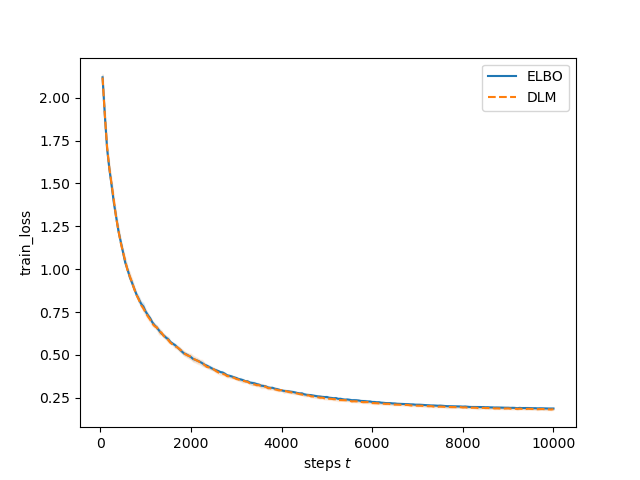}
         \caption{Train loss.}
    \end{subfigure}
    \begin{subfigure}[b]{0.32\textwidth}
         \centering
         \includegraphics[width=\textwidth]{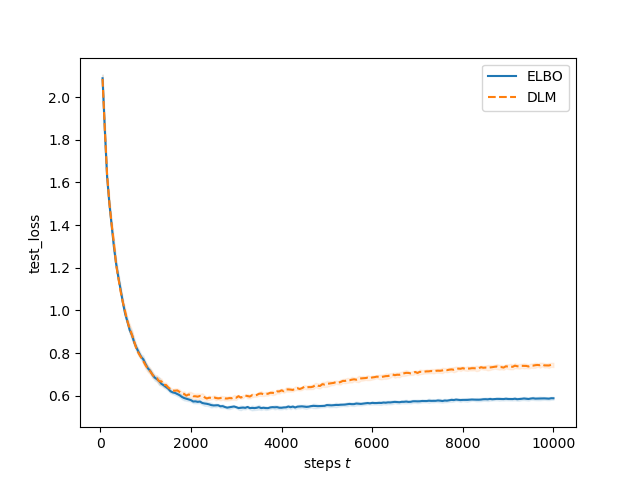}
         \caption{Test loss.}
    \end{subfigure}    
    \begin{subfigure}[b]{0.32\textwidth}
         \centering
         \includegraphics[width=\textwidth]{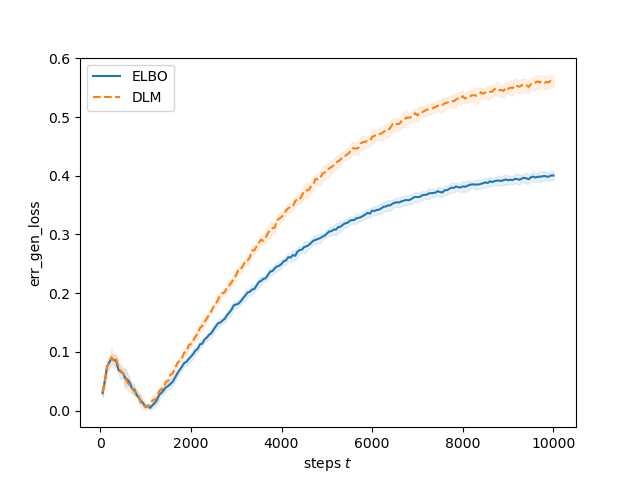}
         \caption{Generalization error.}
    \end{subfigure}
    \begin{subfigure}[b]{0.32\textwidth}
         \centering
         \includegraphics[width=\textwidth]{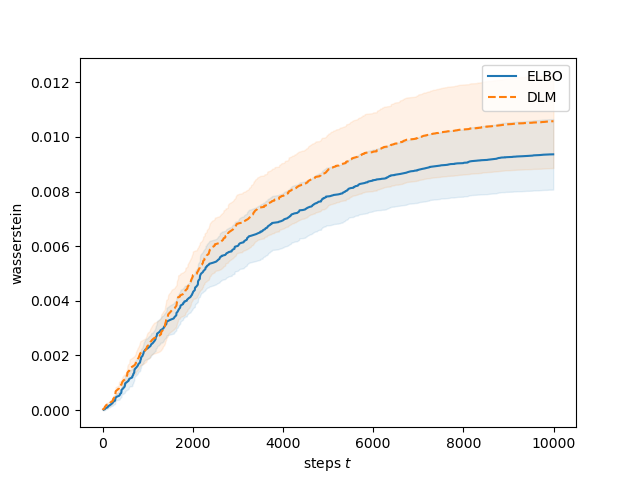}
         \caption{Bound (\ref{eq:wass-bound}, without $K$).}
    \end{subfigure}
    \begin{subfigure}[b]{0.32\textwidth}
         \centering
         \includegraphics[width=\textwidth]{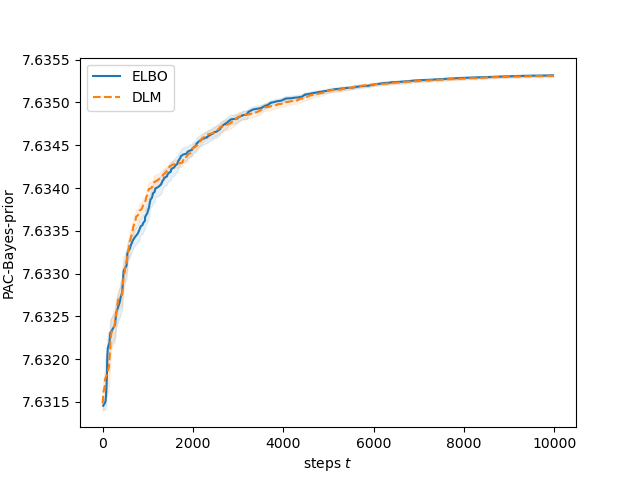}
         \caption{PAC-Bayes (\ref{eq:sqrt}) with prior.}
    \end{subfigure}
    \begin{subfigure}[b]{0.32\textwidth}
         \centering
         \includegraphics[width=\textwidth]{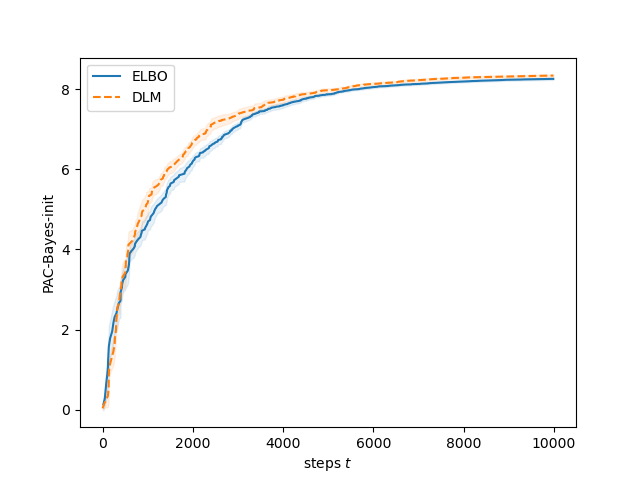}
         \caption{PAC-Bayes (\ref{eq:sqrt}) with init $Q_0$.}
    \end{subfigure}
    \caption{Generalization error and bounds for ELBO and DLM with data augmentation and no random labels.}
    \label{fig:dlm-bound}
\end{figure*}

\paragraph{Generalization bounds: ELBO with data augmentation and random labels.}
To evaluate the bound with parameter differences (\cref{eq:diff}), we need to take expectations over $z$, $\Bar{z}$ and the randomness $\epsilon$. 
To perform this,
we randomly sample 50 pairs of $z$ and $\bar{z}$ from the training and test dataset, respectively. For 
$\epsilon$, we conduct 10 independent runs, with each run selecting a random batch sequence and any other random samples required for optimization.

\cref{fig:bound} (a-c) present the train loss, test loss, and generalization error for ELBO in terms of 0-1 loss along with the stability bound (d) and PAC-Bayes bounds (e,f).
The generalization error is calculated as the absolute difference between the training error and the test error. For the stability bound, we set $C=1$. For PAC-Bayes bounds, we select $\delta=0.025$ and specifically for \cref{eq:bre}, we select $b=100$ and $c=0.1$ following the original paper. 

We first observe that the stability bound is non-vacuous except in the scenario without data augmentation and with 50\% random labels, where there is significant overfitting. The PAC Bayes bounds are less tight in all four scenarios. 
Second, our bound induces the correct ranking over the four cases, and specifically shows that without noisy labels the generalization error is lower when data augmentation is used. 
The PAC Bayes bounds do not demonstrate the benefit of data augmentation in this case. 
Third, note that the smallest generalization error occurs in the case with both data augmentation and 50\% random labels. However, this does not imply the best performance on the test set; in this scenario, the training error converges to 0.5, and the test error is slightly above this value.
Our bound captures this behavior well.

\paragraph{Generalization bounds: ELBO vs.\ DLM.}
\cref{fig:dlm-bound} (a-c) present train and test loss and generalization error in terms of log loss of ELBO vs.\ DLM,
and (d-f) present the 
stability and PAC-Bayes bounds.
When calculating the bound in \cref{eq:wass-bound}, we omit the Lipschitz constant $K$ due to the difficulty in its evaluation. Since the Lipschitz constant remains the same for a given loss function (though not necessarily for the objective), our focus is on the relative comparison between the two methods.
Our bound effectively captures the fact that DLM has a worse generalization error than ELBO. In contrast, the PAC-Bayes bounds are nearly identical for both methods.
Our bound, which is based on the sum of the norms of the gradient differences, underscores the potential instability of the DLM algorithm
for Bayesian neural networks, which might explain its inferior performance for such models.

\section{Conclusion and Future Work}
In this study, we presented a new generalization bound for variational inference by leveraging recent advances in stability-based bounds for Stochastic Gradient Langevin Dynamics (SGLD). Our approach extends the stability argument of stochastic gradient descent to 
a family of algorithms which includes
variational inference, addressing both mean and variance parameters.
Empirical evaluations demonstrated that our bound produces meaningful results with large neural network models and effectively captures generalization error in scenarios involving random labels and data augmentation.

This work opens several promising avenues for future research. The general applicability of our approach suggests that the bound could be extended to various Bayesian algorithms, such as $\text{PAC}^2$ variational learning \citep{Masegosa20}. However, a limitation of our approach is that the bound is primarily effective for algorithms optimized via stochastic gradient descent. For more advanced optimizers like Adam, characterizing parameter differences becomes significantly more challenging.

\section*{Acknowledgments}
This work was partly supported by NSF under grant 2246261. The experiments in this paper were run on the Big Red computing system at Indiana University, supported in part by Lilly Endowment, Inc., through its support for the Indiana University Pervasive Technology Institute.

\FloatBarrier

\bibliography{main}
\bibliographystyle{plainnat}

\appendix
\newpage
\onecolumn

\section{An explanation why the Proof of \citet{Li} does not Apply to Variational Inference}
\label{sec:Li-proof-not-for-vi}

We begin by summarizing the approach taken by \citet{Li} to establish a stability-based generalization bound. Starting from a theorem similar to \cref{thm:errgen}, \citet{Li} bounds the generalization error using the term $2C\E_{S, \Bar{S}} [\tv(Q_S, Q_{\Bar{S}})]$, where $Q_S$ and $Q_{\Bar{S}}$ represent the output distributions of the algorithm $\mathcal{A}$ on datasets $S$ and $\Bar{S}$ after $T$ optimization steps.

Stochastic Gradient Langevin Dynamics (SGLD) updates the parameters by adding isotropic Gaussian noise at each step:
\begin{align*}
    W_t \leftarrow W_{t-1} - \gamma_t g_t(W_{t-1}, B_t) + \sigma_t \mathcal{N}(0, I_d),
\end{align*}
where $g_t(W_{t-1}, B_t)$ denotes the gradient computed on batch $B_t$ at time $t$.

At each step, the distribution of $W_t$ given $W_{t-1}$ is a mixture of Gaussians:
\begin{align*}
    \frac{1}{|\mathcal{B}|} \sum_{B_t \in \mathcal{B}} \mathcal{N}(W_{t-1} - \gamma_t g_t(W_{t-1}, B_T), \sigma_t^2 I).
\end{align*}
Similarly, the distribution of $\bar{W}_t$ on $\Bar{S}$ is
\begin{align*}
    \frac{1}{|\mathcal{B}|} \sum_{B_t \in \mathcal{B}} \mathcal{N}(\Bar{W}_{t-1} - \gamma_t g_t(W_{t-1}, B_t), \sigma_t^2 I).
\end{align*}
This leads to the bound:
\begin{align*}
    \E_{S, \bar{S}}[\tv(Q_S, Q_{\bar{S}})] = \E_{S, \bar{S}}[\tv(W_T, \bar{W}_T)] \leq \E_{S, \bar{S}}\left[\sqrt{\frac{1}{2} \kl(W_T, \bar{W}_T)} \right].
\end{align*}

Applying the chain rule for KL divergence,
\begin{align}
    \kl(W_T, \bar{W}_T) &\leq \kl(W_{1:T}, \bar{W}_{1:T}) \\
    &= \sum_{t=1}^T \E_{w_{1:t-1} \sim W_{1:{t-1}}} \left[\kl(W_t | W_{1:{t-1}} = w_{1:{t-1}}, \bar{W}_t | \bar{W}_{1:{t-1}} = w_{1:{t-1}}) \right]. \label{ineq:kl-chain}
\end{align}
\citet{Li} further bounds the sum of conditional KL divergences using a factor dependent on the difference in gradients evaluated on samples $z$ and $\bar{z}$. As discussed in the main paper, this analysis requires a fixed variance term and is non-obvious to generalize. 

More importantly, 
this strategy does not extend to Variational Inference (VI). 
Recall that for VI at step $t$, the distribution of $W_t$ given past history is parameterized by $\theta_t$.
Hence the evolution of the random variables is over $\theta$ variables and the distribution over $W$ is induced,
i.e., we have 
$(\theta_{t-1}\rightarrow \theta_t)$, $(\theta_{t}\rightarrow W_{t})$, and $(\theta_{t-1}\rightarrow W_{t-1})$.
Due to this structure, the equation which corresponds to 
\cref{ineq:kl-chain} fails when conditioning on $\theta_{1:{t-1}}$ instead of $w_{1:t-1}$.

A counterexample illustrates this failure: let $\theta_1 = 0.4$ and $\bar{\theta}_1 = 0.6$, with $W_1 \sim \text{Bern}(\theta_1)$ and $\bar{W}_1 \sim \text{Bern}(\bar{\theta}_1)$. The parameter updates follow:
\begin{align}
    \theta_2 &= \theta_1 - 0.2, \quad W_2 | \theta_1 \sim \text{Bern}(\theta_1 - 0.2), \label{eq:update1} \\
    \bar{\theta}_2 &= \bar{\theta}_1 - 0.2, \quad \bar{W}_2 | \bar{\theta}_1 \sim \text{Bern}(\bar{\theta}_1 - 0.2). \label{eq:update2}
\end{align}

Then,
\begin{align}
    \kl((W_1, W_2), (\bar{W}_1, \bar{W}_2)) &\approx 0.173, \\
    \kl(W_1, \bar{W}_1) + \E_{\rho \sim \theta_1}[\kl(W_2 | \theta_1 = \rho, \bar{W}_2 | \bar{\theta}_1 = \rho)] &= \kl(W_1, \bar{W}_1) \approx 0.081.
\end{align}
The last equality holds because
$\kl(W_2 | \theta_1 = \rho, \bar{W}_2 | \bar{\theta}_1 = \rho) = 0$ due to the update rules in \cref{eq:update1} and \cref{eq:update2}.
We therefore see that the left-hand side of \cref{ineq:kl-chain} exceeds the right-hand side when conditioning on $\theta_{1:{t-1}}$. This shows that the method of \citet{Li}, using \cref{ineq:kl-chain}, cannot be used for VI.

\section{Omitted Proofs}

\label{sec:proof-bound}

\begin{proof}[Proof of \cref{cor:logT}]
    If $\nabla F(\theta, S, \epsilon)$ is $L$-Lipschitz, then the update rule $G(\theta, S, \epsilon)$ with learning rate $\alpha$ is $(1 + \alpha L)$-expansive:
    \begin{align}
        \lVert G(\theta, S, \epsilon) - G(\theta', S, \epsilon)  \rVert &= \lVert (\theta - \alpha \nabla_\theta F(\theta, S, \epsilon)) - (\theta' - \alpha \nabla_\theta F(\theta', S, \epsilon))\rVert \\
        &\leq \lVert \theta - \theta' \rVert + \alpha \lVert \nabla_\theta F(\theta, S, \epsilon) - \nabla_\theta F(\theta', S, \epsilon) \rVert \\
        &\leq (1 + \alpha L) \lVert \theta - \theta' \rVert.
    \end{align}
    Thus, $\eta_t = 1 + \alpha_t L$. Then
    \begin{align}
        \prod_{i=t+1}^T (1 + \alpha_i L) &\leq \prod_{i=t+1}^T \exp(\alpha_i L) \\
        &= \exp\left(cL\sum_{i=t+1}^T \frac{1}{(i+2) \log (i+2)} \right) \\
        \label{eq:log-int}
        &\leq \exp \left( cL \int_{t+2}^{T+1} \frac{1}{x\log x} dx \right) \\
        &\leq \exp \left(cL (\log \log (T+1) - \log \log (t+2))  \right) \\
        &\leq \frac{\log (T+1)}{\log (t+2)}. \label{eq:logT1}
    \end{align}
    \cref{eq:log-int} is because of the monotonicity of the function $f(x)=\frac{1}{x \log x}$, $\int_{t}^{t+1} \frac{1}{x \log x} dx \geq \frac{1}{(t+1) \log (t+1)}$. 
    In \eqref{eq:intStep} below, we apply the same observation to the function $g(x) = \frac{1}{x \log^2 x}$.
    Using \eqref{eq:logT1} in the bound of \cref{eq:diff} we obtain:
    \begin{align}
        \E\lVert \theta_{T} - \Bar{\theta}_T \rVert &\leq \frac{2\beta \log (T+1)}{n} \sum_{t=1}^T \frac{c}{(t+2) \log^2 (t+2)} \\
        &\leq \frac{2c\beta \log (T+1)}{n} \int_{t=2}^{T+1} \frac{1}{t\log^2 t} dt \label{eq:intStep}\\
        &= \frac{2c\beta \log (T+1)}{n} \left( -\frac{1}{\log t} \Big\vert_{t=2}^T \right) \\
        &\leq \frac{2c\beta \log (T+1)}{n \log 2} = O \left(\frac{\log T}{n} \right).
    \end{align}
\end{proof}


\section{An Extreme Scenario for PAC-Bayes Bounds}
\label{app:pacbayes}

The following example illustrates a situation where the PAC-Bayes bound deteriorates due to its inclusion of the KL term,
and where the stability bound is tight. This complements the qualitative comparison of the bounds in the main paper to illlustrate their strengths and weaknesses. 

Consider a simple logistic regression scenario where the data takes on two possible values, \( x \in \{-1, 1\} \), and the corresponding labels are \( y \in \{0, 1\} \), i.e., there are only two possible examples $(x=-1, y=0)$ and $(x=1, y=1)$. 
The dataset can contain duplicate elements. The log-likelihood in this case is given by:
\begin{align}
    \log p(y \mid w, x) &= -y \log (1 + \exp{(-wx)}) \\
    & \quad - (1-y) \log (1 + \exp(wx)).
\end{align}

Assume we use a Bayesian approach to learn this model, with \( q(w) = \mathcal{N}(m, \sigma^2) \). For simplicity, we assume \(\sigma^2\) is fixed. Recall that for any objective function, we can always evaluate PAC-Bayes bounds.

Suppose our objective is \( F(m, (x, y)) = \mathbb{E}_{q(w)}[-\log p(y \mid w, x)] \). Considering the gradient with respect to \(m\), we have the following identity 
\citep{rezende2014stochastic, opper2008variational}:
\begin{align}
    \nabla_m F(m, (x, y)) &= \mathbb{E}_{q(w)} [\nabla_w \log p(y \mid w, x)].
\end{align}

Observe that:
\begin{align*}
    \nabla_w -\log p(y=1 \mid w, x=1)
    =& \nabla_w \log (1+\exp(-w)) \\
    =& -\frac{\exp{(-w)}}{1+\exp(-w)}, \\
    \nabla_w -\log p(y=0 \mid w, x=-1) 
    =& \nabla_w \log (1+\exp(-w)) \\
    =& -\frac{\exp(-w)}{1+\exp(-w)}, 
\end{align*}
we can see that
\begin{align}
    &\nabla_w -\log p(y=1 \mid w, x=1) \nonumber \\
    = &\nabla_w -\log p(y=0 \mid w, x=-1) < 0. 
    \label{eq:same}
\end{align}
Therefore,
if we run stochastic gradient descent with a constant learning rate for sufficiently many steps, we reach a solution where \(m \rightarrow +\infty\).

Now, suppose the initial prior is \( P_0 = \mathcal{N}(0, \sigma^2) \). The KL divergence will eventually become:
\begin{align}
    \kl(\mathcal{N}(m, \sigma^2) \parallel \mathcal{N}(0, \sigma^2)) &= \frac{m^2}{2\sigma^2} \rightarrow +\infty.
\end{align}

However, if we consider the stability bound, which is based on the gradient difference, the situation changes. It’s clear that if \( z = \bar{z} \) (whether \( x = \bar{x} = 1, y = \bar{y} = 1 \) or \( x = \bar{x} = -1, y = \bar{y} = 0 \)), the gradient difference will be zero. Thus, we only need to consider the case where \( z = (1, 1) \) and \( \bar{z} = (-1, 0) \). As shown in \cref{eq:same}, the gradients are the same in this scenario as well.

Therefore, using the stability bound, the generalization error will be zero. In contrast, the PAC-Bayes bound gives a value of \( \infty \).



\end{document}